\documentclass[twoside]{article}

\usepackage{microtype}
\usepackage{graphicx}
\usepackage{booktabs}
\usepackage{bbold}
\usepackage{multirow}
\usepackage{geometry}
\usepackage{parskip}
\usepackage[ruled,vlined]{algorithm2e}
\usepackage{authblk}
\usepackage{enumitem}
\usepackage{subcaption}

\usepackage{xcolor}      
\usepackage{amssymb}     

\newcommand{\cmark}{\textcolor{green!70!black}{\checkmark}}
\newcommand{\xmark}{\textcolor{red}{\ding{55}}} 

\usepackage{pifont} 

\usepackage[colorlinks,citecolor=blue,linkcolor=blue,urlcolor=blue]{hyperref}

\usepackage{amsmath}
\usepackage{amssymb}
\usepackage{mathtools}
\usepackage{amsthm}

\usepackage[capitalize,noabbrev]{cleveref}

\theoremstyle{plain}
\newtheorem{theorem}{Theorem}[section]
\newtheorem{proposition}[theorem]{Proposition}

\theoremstyle{definition}
\newtheorem{definition}[theorem]{Definition}

\theoremstyle{remark}
\newtheorem{remark}[theorem]{Remark}

\newcommand{\methodname}[1]{\textsf{#1}\xspace}

\usepackage[accepted]{aistats2026}
\usepackage[round]{natbib}

\begin{document}

%

%

\twocolumn[

\aistatstitle{Adaptive Coverage Policies in Conformal Prediction}

\aistatsauthor{ Etienne Gauthier \And Francis Bach \And  Michael I. Jordan }

\aistatsaddress{ Inria, Ecole Normale Supérieure,\\ PSL Research University \And  Inria, Ecole Normale Supérieure,\\ PSL Research University \And Inria, Ecole Normale Supérieure,\\ PSL Research University,\\ University of California, Berkeley } ]

\begin{abstract}
  Traditional conformal prediction methods construct prediction sets such that the true label falls within the set with a user-specified coverage level. However, poorly chosen coverage levels can result in uninformative predictions, either producing overly conservative sets when the coverage level is too high, or empty sets when it is too low. Moreover, the fixed coverage level cannot adapt to the specific characteristics of each individual example, limiting the flexibility and efficiency of these methods. In this work, we leverage recent advances in e-values and post-hoc conformal inference, which allow the use of data-dependent coverage levels while maintaining valid statistical guarantees. We propose to optimize an adaptive coverage policy by training a neural network using a leave-one-out procedure on the calibration set, allowing the coverage level and the resulting prediction set size to vary with the difficulty of each individual example. We support our approach with theoretical coverage guarantees and demonstrate its practical benefits through a series of experiments.
\end{abstract}

\section{INTRODUCTION}

Conformal prediction \citep{gammerman1998learning,vovk2005algorithmiclearning,shafer2008tutorial} is a powerful framework for quantifying uncertainty that is particularly useful in machine-learning applications \citep{papadopoulos2002inductivecm,balasubramanian2014conformal,laxhammar2015inductive,lei2018inference,chernozhukov2018exact,angelopoulos2021uncertainty,fisch2021few,cella2021valid,johnstone2021conformal,bates2023testing,su2024llm}. It provides prediction sets that contain the true label with high probability, without relying on parametric assumptions about the data distribution. The only requirement is exchangeability, meaning the joint distribution of the data remains invariant under any permutation of the samples. This assumption is significantly weaker than the standard independent and identically distributed (i.i.d.) assumption, making conformal prediction broadly applicable while still offering strong, distribution-free guarantees.

In standard conformal prediction, the available data is typically split into a training set, used to fit a predictive model, and a calibration set, which is held out to assess the model’s behavior on unseen data. Applying the model to the calibration set yields nonconformity scores, which quantify how unusual each true label is relative to the model’s prediction. These scores are then used to construct prediction sets for a new test point that satisfy the desired coverage guarantee. Formally, letting~$\mathcal{X}\times\mathcal{Y}$ denote the feature-label space, we assume that a calibration set $\{(X_i, Y_i)\}_{i=1}^n$ and a new input~$(X_{\rm test},Y_{\rm test})$ have been drawn exchangeably from some distribution $\mathbb{P}$ over $\mathcal{X}\times\mathcal{Y}$. Conformal prediction constructs a prediction set, $\hat{C}_n^\alpha(X_{\rm test}) \subseteq \mathcal{Y}$, based on the calibration set such that
\begin{equation}
\label{cp}
\mathbb{P}(Y_{\text{test}} \in \hat{C}_n^\alpha(X_{\text{test}})) \ge 1 - \alpha,
\end{equation}
where $\alpha \in (0,1)$ is a user-specified miscoverage level. This guarantee is marginal, holding over all sources of randomness including both the calibration samples and the test point.

Machine learning prediction methods  often focus on producing a point prediction for a new input $X$, via a mapping $f \colon \mathcal{X} \to \mathcal{Y}$, without indicating uncertainty in the prediction. Conformal prediction addresses this limitation by providing a principled, model-agnostic approach to uncertainty quantification: it acts as a wrapper around $f$, transforming its outputs into a prediction set that contains the true label with high probability. To construct such sets, conformal prediction uses a score function, $S \colon \mathcal{X} \times \mathcal{Y} \to \mathbb{R}$, that depends on the model~$f$, and quantifies how well a candidate label matches the model’s prediction. The standard approach in conformal prediction relies on the observation that the calibration scores $S(X_i, Y_i)$ for $i = 1, \ldots, n$ and the test score $S(X_{\rm test}, Y_{\rm test})$ are exchangeable. As a result, their ranks follow a uniform distribution. In particular, the test score is unlikely to rank among the highest, which is the key insight used to construct valid prediction sets in conformal prediction. For a comprehensive treatment of conformal prediction and its applications in machine learning, we refer the reader to \cite{angelopoulos2023gentle} and \cite{angelopoulos2024theoreticalfoundationsconformalprediction}.

A core feature of classical conformal prediction is the specification of a fixed coverage level $1-\alpha$ prior to observing the data. While this provides a highly desirable and easily interpretable statistical guarantee, it imposes a rigid operational constraint. In practice, this fixed choice can lead to uninformative predictions: small values of $\alpha$ can yield overly conservative prediction sets, while large values may lead to empty sets. Moreover, classical conformal methods do not allow $\alpha$ to be chosen based on the observed data, because the marginal coverage guarantee (\ref{cp}) only holds when $\alpha$ is fixed before seeing the calibration or test points. In many operational settings, a practitioner might prefer to adjust $\alpha$ after inspecting the data. For instance, in a medical diagnosis task, a standard conformal method guaranteeing 99\% coverage might output a broad set of seven diagnoses. While statistically valid, this set may be clinically unactionable. A physician might prefer to dynamically relax the coverage slightly (e.g., from $\alpha=1\%$ to $\alpha=2\%$) to yield a precise, actionable set of three diagnoses. Most practitioners would naturally opt for the latter, as it provides a more informative and usable prediction. Unfortunately, in traditional conformal prediction, choosing $\alpha$ after inspecting the data undermines the marginal coverage guarantee (\ref{cp}) and is closely related to the phenomenon of p-hacking in the statistical literature \citep{simmons2011false,head2015phacking}. This rigidity is limiting in settings where uncertainty varies across examples, or where we may want to tailor coverage to the difficulty of each instance.

The method proposed by \cite{cherian2024llmvalidity} uses a neural network to fit a data-dependent miscoverage~$\tilde{\alpha}$ in conformal prediction. However, their network is trained on separate training data and offers no guidance on how to optimize training to achieve a desired expected prediction set size. By contrast, our approach trains $\tilde{\alpha}$ directly on calibration data, without needing to reserve any training data for this purpose, and allows practitioners to adjust the training procedure to control the expected size of prediction sets at test time.

We tackle this challenge by leveraging recent advances in e-values \citep{shafer2019gametheoretic,vovk2021evalues,grunwald2024safetesting,ramdas2024hypothesistestingevalues} and post-hoc inference \citep{wang2022fdr, xu2024postselectioninference,grunwald2024beyond,koning2024posthocalphahypothesistesting,gauthier2025evaluesexpandscopeconformal,chugg2025admissibility}, which provide valid coverage guarantees even when the miscoverage level~$\alpha$ is selected adaptively. It is known that conformal sets constructed using e-values generally yield slightly larger prediction sets than those built with standard p-values for a given fixed $\alpha$ \citep{Vovk2025conformaleprediction}. However, e-values are necessary to enable this adaptivity as they are the only way to construct post-hoc p-values, the statistical objects that allow for valid inference even when the miscoverage level is chosen after observing the data \citep{koning2024posthocalphahypothesistesting}. We willingly trade a slight loss in \mbox{fixed-$\alpha$} efficiency for this powerful operational advantage, as it allows the coverage level to be selected adaptively without invalidating the statistical guarantees.

Building on this, we propose to optimize an adaptive coverage policy by training a neural network on observed data. To construct training examples, we adopt a leave-one-out approach on the calibration set: for each held-out point, the remaining points serve as a pseudo calibration set, and the held-out point acts as a pseudo test sample. This generates a collection of labeled examples, enabling a model to predict a sample-specific coverage policy. Our aim is to design a mapping from each pseudo calibration-test pair to a coverage policy that maximizes the \emph{informativeness} of the resulting prediction sets. We formalize this informativeness as a dual objective (detailed in Section \ref{sub:informativeness}): minimizing the prediction set size while simultaneously minimizing the adaptive miscoverage level $\tilde{\alpha}$. This trade-off is governed by a user-specified regularization parameter $\lambda$. Furthermore, in Section \ref{sub:lambda}, we introduce a principled procedure to select $\lambda$ such that the expected prediction set size at test time meets a desired target. Crucially, this approach allows the data-dependent miscoverage level $\tilde{\alpha}$ to adapt to individual test samples while maintaining valid marginal guarantees, offering a more flexible and data-driven alternative to choosing a fixed $\alpha$ in conventional conformal prediction.

By leveraging the post-hoc validity of e-values, this approach enables adaptive conformal prediction, modulating prediction set size according to instance difficulty. Our work is related to recent efforts that combine e-values with conformal prediction \citep{Vovk2025conformaleprediction, gauthier2025evaluesexpandscopeconformal}. In particular, we contrast our approach with recent prior work \citep{gauthier2025backwardconformalprediction,liu2026stbcptighteningcoveragebound}, which also utilizes e-values and a leave-one-out strategy. However, our objective differs fundamentally: while that method inverts the conformal procedure to find the minimum $\alpha$ satisfying a \emph{hard} size constraint for a specific prediction set, our method learns a parametric coverage policy to optimize a \emph{soft} constraint by targeting an expected prediction set size at test time. Furthermore, the two methods employ the leave-one-out procedure for entirely different purposes. Whereas the prior work uses it solely to estimate coverage guarantees, we utilize it as an integral generative step to construct the training data required to fit our neural coverage policy.

\section{METHOD}

Traditional conformal prediction methods compare the rank of the test score to those in the calibration set, a procedure that can be interpreted in terms of p-values. As we will discuss, an alternative is to base conformal inference on e-values. This approach has a wider range of applicability and in particular will permit us to obtain valid inference even when the miscoverage level~$\alpha$ is selected in a data-dependent manner.

\subsection{Conformal e-prediction}

Conformal sets can be constructed using e-values, a method known as conformal e-prediction. E-values are simply the realizations of random variables known as e-variables:

\begin{definition}[E-variable]
    An \emph{e-variable} E is a nonnegative random variable that satisfies 
    \[
    \mathbb{E}[E] \le 1.
    \]
\end{definition}

Thresholding an e-variable at level $1/\alpha$ yields a prediction set with marginal coverage at least $1 - \alpha$. Indeed, we can apply Markov's inequality to obtain:

\[
\mathbb{P}(E < 1/\alpha) \ge 1 - \alpha.
\]

While conformal e-prediction is compatible with any valid e-variable, for concreteness we employ the \emph{soft-rank e-variable}. This construction, which first appeared in \citet{wang2022fdr} and \citet{koning2025measuring}, and was later applied in the context of conformal prediction by \citet{balinsky2024EnhancingCP}, takes the following form:
\begin{equation}
\label{def_e_var}
E = \frac{S(X_{\rm test},Y_{\rm test})}{\frac{1}{n+1}\left(\sum_{i=1}^{n}S(X_i,Y_i)+S(X_{\rm test},Y_{\rm test})\right)}.
\end{equation}

This quantity defines a valid e-value as long as the scores are exchangeable and non-negative. Note that it is meaningful only when the score function $S$ is negatively oriented; that is, lower scores indicate better predictions. In the remainder of this paper, we assume that these conditions hold. 

Intuitively, the soft-rank e-variable construction mirrors the logic of traditional conformal prediction: the test score cannot be disproportionately large relative to the average calibration scores. However, unlike traditional conformal prediction that uses rank-based comparisons, the soft-rank e-variable directly compares the actual score values.

Using e-values goes beyond simply applying Markov’s inequality to obtain a valid conformal set with a fixed miscoverage level $\alpha$. They possess stronger properties, such as post-hoc guarantees, which allow for coverage guarantees even when the significance level $\alpha$ is chosen based on the observed data, something standard conformal prediction methods cannot provide.

\subsection{Post-hoc Validity}

We recall the key result on post-hoc validity with e-variables in conformal prediction, stated here for the specific case of the soft-rank e-variable.

\begin{proposition}[\citet{gauthier2025evaluesexpandscopeconformal}]
\label{prop:post-hoc-cp}
    Consider a calibration set $\{(X_i,Y_i)\}_{i=1}^n$ and a test data point $(X_{\rm test},Y_{\rm test})$ such that $(X_1,Y_1),\dotsc,(X_n,Y_n),(X_{\rm test},Y_{\rm test})$ are exchangeable. Let $\tilde{\alpha} > 0$ be any miscoverage level may depend on all of these data points. Then we have that:
\begin{equation}
\label{eq:posthoc}
    \mathbb{E}\bigg[ \frac{\mathbb{P} (Y_{\rm test} \not\in \hat{C}_n^{\tilde{\alpha}}(X_{\rm test})\mid \tilde{\alpha})}{\tilde{\alpha}} \bigg] \le 1,    
\end{equation}
    where 
\begin{equation}
\label{conformal_set}
    \hat{C}_n^{\tilde{\alpha}}(x)\!:=\! \left\{y\!:\!\frac{S(x,y)}{\frac{1}{n+1}\left(\sum_{i=1}^{n}S(X_i,Y_i) + S(x,y) \right)}\!<\!\frac{1}{\tilde{\alpha}} \right\}.
\end{equation}
When $\tilde{\alpha}$ is a fixed constant independent of the data, the guarantee (\ref{eq:posthoc}) reduces to the standard conformal guarantee (\ref{cp}).
\end{proposition}

Proposition \ref{prop:post-hoc-cp} enables us to obtain marginal guarantees for any coverage level, including those that depend on the data. This post-hoc validity holds because the definition of the e-value relies strictly on the non-conformity scores, which are exchangeable by assumption. Because an adaptive policy does not alter the computation of these underlying scores, selecting the miscoverage level $\tilde{\alpha}$ in a data-dependent manner does not retroactively compromise exchangeability. Building on this flexibility, we aim to design a coverage policy that adapts the miscoverage level in order to minimize the size of the resulting prediction sets. Informally, a coverage policy is simply a rule that maps the calibration scores and potential test scores to a (data-dependent) miscoverage level $\tilde\alpha \in (0,1)$. We will provide a formal definition in Definition~\ref{def:coverage-policy}, but before doing so, we first examine how prediction set sizes behave in both the classification and regression settings. This detour is useful since we want to design coverage policies that minimize these sizes, and the form they take in these two cases will guide us toward a simplified definition of coverage policies.

\begin{remark}[Conformal set size in classification]
\label{rk:classif}
   Given a calibration set $(X_i,Y_i)_{i=1}^n$ and a test feature $X_{\rm test}$, in classification problems the size of a conformal set~$\rm Size \big(\hat{C}_n^{\tilde{\alpha}}(X_{\text{test}})\big)$ at miscoverage level $\tilde{\alpha} \in (0,1)$ is  given by
\[
\#\left\{ y \in \mathcal{Y} : \frac{S(X_{\rm test},y)}{\tfrac{1}{n+1}\!\left(\sum_{i=1}^nS(X_i,Y_i)+S(X_{\rm test},y)\right)}\! <\! \frac{1}{\tilde{\alpha}}\right\}.
\]
Because the cardinality operator is non-differentiable, it cannot be used directly in gradient-based optimization methods typical in machine learning. Since our ultimate goal is to train coverage policies via gradient-based methods, we replace the indicator inside the cardinality with a smooth surrogate. Specifically, we leverage the approximation:
\begin{align}
\label{eq:smooth_sigmoid}
&\sum_{y \in \mathcal{Y}} \mathbb{1}{\left\{\frac{S(X_{\rm test},y)}{\frac{1}{n+1}\left(\sum_{i=1}^n S(X_i,Y_i)+S(X_{\rm test},y)\right)} < \frac{1}{\tilde{\alpha}} \right\}} \notag \\
&\!\approx\! \sum_{y \in \mathcal{Y}} \!\sigma \! \left(\! k \! \left( \frac{1}{\tilde{\alpha}}\! - \!\frac{S(X_{\rm test},y)}{\tfrac{1}{n+1}\left(\sum_{i=1}^n S(X_i,Y_i)+S(X_{\rm test},y) \right)}\! \right)\! \right),
\end{align}
where $\sigma(x) = 1/(1 + e^{-x})$ is the sigmoid function, and~$k > 0$ is a parameter controlling the sharpness of the approximation. As $k \to \infty$, the sigmoid approaches a step function, and the approximation becomes exact. This smooth approximation enables efficient end-to-end training using standard neural network toolkits. It is worth noting, however, that alternative approaches to sigmoid smoothing, such as randomized smoothing proposed by \cite{berthet2020perturbed}, may lead to more robust optimization.
\end{remark}

\begin{remark}[Conformal set size in regression]
The conformal set built using the soft-rank e-variable can be rewritten as:
\begin{align*}
\hat{C}_{n}^{\tilde{\alpha}}(X_{\rm test}) 
&= \left\{y \in \mathcal{Y} : S(X_{\rm test},y) < \frac{\sum_{i=1}^nS(X_i,Y_i)}{(n+1)\tilde{\alpha}-1} \right\},
\end{align*}
using basic algebraic simplifications. Now, consider a standard choice of score in regression, namely the mean absolute error (MAE):
\[
S(x,y) = |f(x)-y|.
\]
In this case, one can directly express the size of the conformal set,
\begin{equation}
\label{eq:size_reg}
\rm Size \big( \hat{C}_{n}^{\tilde{\alpha}}(X_{\rm test}) \big) = 2 \frac{\sum_{i=1}^nS(X_i,Y_i)}{(n+1)\tilde{\alpha}-1},
\end{equation}
provided that $\tilde{\alpha} > 1/(n+1)$. This expression is differentiable almost everywhere with respect to $\tilde\alpha$, enabling gradient-based optimization.
\label{rmk:sets_reg}
\end{remark}

We observe that the prediction sets considered here, based on the soft-rank e-value conformal set~(\ref{conformal_set}), depend on the calibration scores in a very simple way: they depend only on the sum of all calibration scores. In addition, in classification, the size also depends on the vector of potential test scores $(S(X_{\rm test},y))_{y \in \mathcal{Y}}$, while in regression the size does not depend on the test feature at all.  

This observation motivates a natural simplification: we can define a coverage policy directly as a function of the sum of the calibration scores and an appropriate test summary statistic (which may be empty in the regression setting) as follows.

\begin{definition}[Coverage policy]
\label{def:coverage-policy}
A \emph{coverage policy} is a function
\[
\pi: \mathbb{R}_+ \times \mathcal{T} \to (0,1),
\]
that maps the sum of calibration scores and a test summary statistic to a miscoverage level.\footnote{While we refer to it as a coverage policy, the function actually outputs a miscoverage level $\tilde{\alpha}$. This convention simplifies our notation, since the conformal set is defined directly using the inverse $1/\tilde{\alpha}$.} Here, $\mathcal{T}$ denotes the space of possible test statistics $t(X_{\rm test})$: in classification, $t(X_{\rm test})$ is the vector of potential scores~$(S(X_{\rm test},y))_{y\in\mathcal{Y}}$, whereas in regression it is empty, reflecting that the conformal set size (\ref{eq:size_reg}) does not depend on the test feature. The output 
\[
\pi\big(\textstyle\sum_{i=1}^n S(X_i,Y_i), t(X_{\rm test})\big)
\]
specifies the miscoverage level $\tilde{\alpha}$ to be used for constructing the conformal set $\hat{C}_{n}^{\tilde{\alpha}}(X_{\rm test})$.
\end{definition}

Definition \ref{def:coverage-policy} formalizes the idea that the miscoverage level can adapt based on summary information from the calibration set and the new test point, while remaining compatible with the guarantees of Proposition~\ref{prop:post-hoc-cp}. Our goal is to design a coverage policy that selects the miscoverage level to produce prediction sets that are as informative as possible. 

\subsection{Training a Coverage Policy via a Leave-One-Out Procedure}
\label{sub:informativeness}

For concreteness, we model the coverage policy using a neural network
\[
\tilde{\alpha}_\theta: \mathbb{R}_+ \times \mathcal{T} \to (0,1),
\]
parameterized by $\theta \in \Theta$ (though any model could be used). The goal of training is to select parameters $\theta$ that produce informative prediction sets while maintaining appropriate coverage.

To generate training samples, we employ a leave-one-out procedure: for a given calibration set of size $n$ and~$j \in \{1,\dots,n\}$, we treat the $j$-th sample as a pseudo test point and the remaining $n-1$ samples as a pseudo calibration set. Repeating this for all $n$ points in the set produces $n$ pseudo calibration-test pairs, each of which serves as a labeled example for training. The intuition behind this procedure is that, when leaving out a single data point, the aggregated information from the remaining~$n-1$ points changes only slightly. As a result, each pseudo calibration-test pair provides nearly the same perspective as if we were observing a fresh sample with a fresh calibration set together with a new test point. This allows the network to infer a meaningful mapping from calibration sets to coverage levels without needing multiple independent sets.

For the $j$-th training example of the leave-one-out procedure, we denote the network’s predicted miscoverage~by
\begin{equation}
\label{alpha_tilde_j}
\tilde{\alpha}_\theta^j := \tilde{\alpha}_\theta\big(\textstyle\sum_{i\neq j} S(X_i,Y_i), t(X_j)\big),
\end{equation}
and for the actual test point we write
\begin{equation}
\label{alpha_tilde_test}
\tilde{\alpha}_\theta^{\rm test} := \tilde{\alpha}_\theta\big(\textstyle\sum_{i=1}^n S(X_i,Y_i), t(X_{\rm test})\big).
\end{equation}
Note that at training time the network is fed sums of~$n-1$ scores, while at test time it receives a sum over $n$ scores. The difference is negligible: leaving out one observation only slightly perturbs the aggregate, so the pseudo examples are essentially indistinguishable from the true test-time scenario.

We aim to tune $\theta$ so as to minimize the size of the prediction sets produced across the $n$ training samples. However, naively minimizing prediction set size leads to a degenerate solution with coverage equal to zero. To avoid this, we introduce a regularization term that penalizes overly large miscoverage, thereby stabilizing the training process and discouraging degenerate solutions. The strength of this penalty is controlled by a user-specified parameter $\lambda > 0$, which allows practitioners to balance between the compactness of the prediction sets and the conservativeness of the coverage level. A smaller $\lambda$ places more weight on obtaining smaller prediction sets, while a larger $\lambda$ emphasizes achieving better coverage.

Specifically, we train $\tilde{\alpha}_\theta$ by minimizing the following objective:
\begin{equation}
\label{loss}
\mathcal{L}_\lambda(\theta) = \frac{1}{n} \sum_{j=1}^n \rm Size \big( \hat{C}_{n-1}^{\tilde{\alpha}_\theta^j}(X_j) \big) + \lambda \cdot \tilde{\alpha}_\theta^j,
\end{equation}
where $\hat{C}_{n-1}^{\tilde{\alpha}_\theta^j}(X_j)$ is the prediction set built from the pseudo calibration set obtained by leaving out the $j$‑th sample, applied to the pseudo test feature $X_j$, with miscoverage level $\tilde{\alpha}_\theta^j$ predicted by the network. The first term encourages informative prediction sets, while the second term discourages the network from selecting excessively high miscoverage levels.

\begin{algorithm}[ht]
\caption{Training a Coverage Policy via Leave-One-Out}
\label{algorithm}
\KwIn{Calibration set $\{(X_1, Y_1), \dots, (X_n, Y_n)\}$, score function $S$, untrained neural network $\tilde{\alpha}_\theta$, regularization parameter $\lambda$, batch size~$B$, optimizer}
\KwOut{Trained neural network $\tilde{\alpha}_\theta$}
\BlankLine
\textbf{Step 1: Construct pseudo episodes via leave-one-out}  

\For{$j = 1, \dots, n$}{
    Define pseudo test point $X_j$ and pseudo calibration set $\{(X_i, Y_i) : i \neq j\}$
}

\BlankLine
\textbf{Step 2: Update network parameters}  

Initialize $\theta$ randomly\;

\While{not converged}{
    Sample a minibatch $\mathcal{B} \subset \{1, \dots, n\}$ of size $B$\;
    
    \ForEach{$j \in \mathcal{B}$}{
        Compute coverage level $\tilde{\alpha}_\theta^j$ using (\ref{alpha_tilde_j});
        
        Construct conformal set $\hat{C}_{n-1}^{\tilde{\alpha}_\theta^j}(X_j)$ defined in (\ref{conformal_set}) using score function $S$\;
        
        Compute $\rm Size\big(\hat{C}_{n-1}^{\tilde{\alpha}_\theta^j}(X_j)\big)$\;
    }
    
    Compute loss defined in (\ref{loss}):
    \[
    \mathcal{L}_\lambda(\theta) = \frac{1}{B} \sum_{j \in \mathcal{B}} \rm Size\big(\hat{C}_{n-1}^{\tilde{\alpha}_\theta^j}(X_j)\big) + \lambda \cdot \tilde{\alpha}_\theta^j
    \]
    
    Update $\theta$ by minimizing $\mathcal{L}_\lambda(\theta)$ using the optimizer\;
}

\BlankLine
\Return $\tilde{\alpha}_\theta$
\end{algorithm}

By minimizing the loss (\ref{loss}) in the leave-one-out protocol, we fit a coverage policy that adaptively selects a miscoverage level $\tilde\alpha$ at test time. We summarize the training procedure of $\tilde\alpha_\theta$ in Algorithm~\ref{algorithm}. Using the neural network output by Algorithm \ref{algorithm}, we can compute the test-time miscoverage $\tilde{\alpha}_\theta^{\rm test}$ as defined in (\ref{alpha_tilde_test}). The associated conformal set $\hat{C}_{n}^{\tilde{\alpha}_\theta^{\rm test}}\!(X_{\rm test})$ is optimized for informative prediction sets while still satisfying the marginal coverage guarantee (\ref{eq:posthoc}), ensuring both practical efficiency and rigorous statistical reliability.

\subsection{Selecting $\lambda$: Insights From the Constant-$\alpha$ Setting}
\label{sub:lambda}

The choice of the regularization term $\lambda$ is critical in practice, as it directly impacts, among other things, the expected size of the prediction set at test time. One practical strategy to select $\lambda$ is to monitor the behavior of the network $\tilde\alpha_\theta$ during training and track the final average sizes of the training conformal sets. 

To build intuition for this strategy, we first analyze an idealized setting in which the network output is constant, i.e., $\tilde\alpha_\theta \equiv \alpha$. In this case, we can show that the average size under the leave-one-out protocol provides an accurate estimate of the expected test-time prediction set size. More precisely, the estimation error is of order $O_P\left(1/\sqrt{n}\right)$ as the calibration size $n$ increases (using the $O_P$ notation from \citet{vaart1998asymptotics}).

\begin{theorem}[Leave-one-out proxy under constant~$\alpha$]
\label{thm:loo-size-consistency}
Assume that the calibration samples $(X_i,Y_i)$ are i.i.d., and let $\alpha \in (0,1)$ be a given target miscoverage level.  

Assume one of the following two cases holds:
\begin{enumerate}[label=(\roman*)]
\item \textbf{(Classification, sigmoid smoothing)}  
The size is given by the smooth sigmoid approximation defined in (\ref{eq:smooth_sigmoid}) with some parameter $k>0$.  
The score function $S$ is bounded and takes values in $[S_{\min},S_{\max}]$ with $0<S_{\min}\le S_{\max}<\infty$, and $n > S_{\max}/S_{\min}$.  

\item \textbf{(Regression, MAE score)}  
The size is defined in (\ref{eq:size_reg}).  
The score function $S$ is bounded and takes values in $[0,S_{\max}]$ with $S_{\max}<\infty$, and the miscoverage level satisfies $\alpha > 1/n$.  
\end{enumerate}

Let $\overline{\rm Size}_n := \frac{1}{n}\sum_{j=1}^n \rm Size(\hat{C}_{n-1}^{\alpha}(X_j))$ denote the average size under the leave-one-out protocol. Then, under either (i) or (ii),  
\[
\Big|\overline{\rm Size}_n - \mathbb{E}\left[\rm Size\big(\hat{C}_n^{\alpha}(X_{\rm test})\big)\right]\Big| = O_P\!\left(\frac{1}{\sqrt{n}}\right),
\]
i.e., the average size consistently estimates the expected test-time size at rate \(1/\sqrt{n}\) in probability.
\end{theorem}

Theorem~\ref{thm:loo-size-consistency}, proved in Appendix~\ref{appendix:proof}, shows that in the idealized constant-output case, the average size of the conformal sets provides a reliable estimate of the expected size at test time. This result provides a theoretical foundation for our approach: monitoring the leave-one-out sizes during training to approximate the expected test-time behavior of the conformal predictor.

In practice, the network output is generally a trained, input-dependent function rather than a constant. While Theorem \ref{thm:loo-size-consistency} does not directly extend to this more realistic case, we observe empirically that the same approximation remains accurate: when the network varies smoothly with the input distribution, the leave-one-out average continues to track the expected test-time size. Our experiments in the next section substantiate this observation.\footnote{Our code is publicly available at \url{https://github.com/GauthierE/adaptive-coverage-policies}.}

\begin{algorithm}[ht]
\caption{Two-Stage $\lambda$-Selection: Bracketing then Bisection}
\label{algorithm_lambda}
\KwIn{Calibration set $\{(X_1, Y_1), \dots, (X_n, Y_n)\}$, score function $S$, target size $M>0$, tolerance $\varepsilon>0$, untrained neural network $\tilde{\alpha}_\theta$, initial $\lambda>0$}
\KwOut{$\lambda_M$ such that $\mathbb{E}\big[\rm Size\big(\hat{C}_n^{\tilde\alpha_\theta^{\rm test}}\!(X_{\rm test})\big)\big] \approx M$}
\BlankLine
\textbf{Step 1: Expansion phase to bracket $M$}  

\Repeat{$\overline{\rm Size}_n(\lambda)$ crosses $M$}{
    Train $\tilde{\alpha}_\theta$ with parameter $\lambda$ using Algorithm~\ref{algorithm}\;
    Compute $\overline{\rm Size}_n(\lambda) := \frac{1}{n}\sum_{j=1}^n \rm Size(\hat{C}_{n-1}^{\tilde\alpha_\theta^j}(X_j))$\;
    \uIf{$\overline{\rm Size}_n(\lambda) < M$}{
        $\lambda \leftarrow 2\lambda$\;
    }
    \Else{
        $\lambda \leftarrow \lambda/2$\;
    }
}
Set $\lambda_{\rm low}$ and $\lambda_{\rm high}$ as the two most recent values of $\lambda$ bracketing $M$\;

\BlankLine
\textbf{Step 2: Bisection refinement}  

\Repeat{$\big|\overline{\rm Size}_n(\lambda) - M\big| \le \varepsilon$}{
    $\lambda \leftarrow (\lambda_{\rm low} + \lambda_{\rm high})/2$\;
    Train $\tilde{\alpha}_\theta$ with parameter $\lambda$ using Algorithm~\ref{algorithm}\;
    Compute $\overline{\rm Size}_n(\lambda)$\;
    \uIf{$\overline{\rm Size}_n(\lambda) < M$}{
        $\lambda_{\rm low} \leftarrow \lambda$\;
    }
    \Else{
        $\lambda_{\rm high} \leftarrow \lambda$\;
    }
}

\BlankLine
\Return{$\lambda_M \leftarrow \lambda$}
\end{algorithm}

Building on this intuition, Algorithm~\ref{algorithm_lambda} describes our procedure for selecting the regularization parameter~$\lambda$ given a desired target average prediction set size~$M$. The goal is to find a $\lambda$ such that the empirical average size under the leave-one-out protocol is close to~$M$, which, by Theorem~\ref{thm:loo-size-consistency}, implies that the expected test-time size will also be close to~$M$.

A key ingredient in this procedure is the monotonicity of the leave-one-out size with respect to $\lambda$. In the \mbox{constant-$\alpha$} setting, Proposition~\ref{prop:monotone-size}, proved in Appendix~\ref{appendix:proof-monotone}, establishes this property formally. Monotonicity ensures that we can safely use a bracketing-and-bisection strategy: once we identify two values of $\lambda$ such that the leave-one-out average lies below and above~$M$, repeated halving of the interval guarantees convergence to the desired value. Empirically, we observe that this monotonicity approximately holds in the trained, input-dependent setting, ensuring the reliability of the bracketing-and-bisection procedure in practice.

\begin{figure*}[h!]
    \centering
    \includegraphics[width=\textwidth]{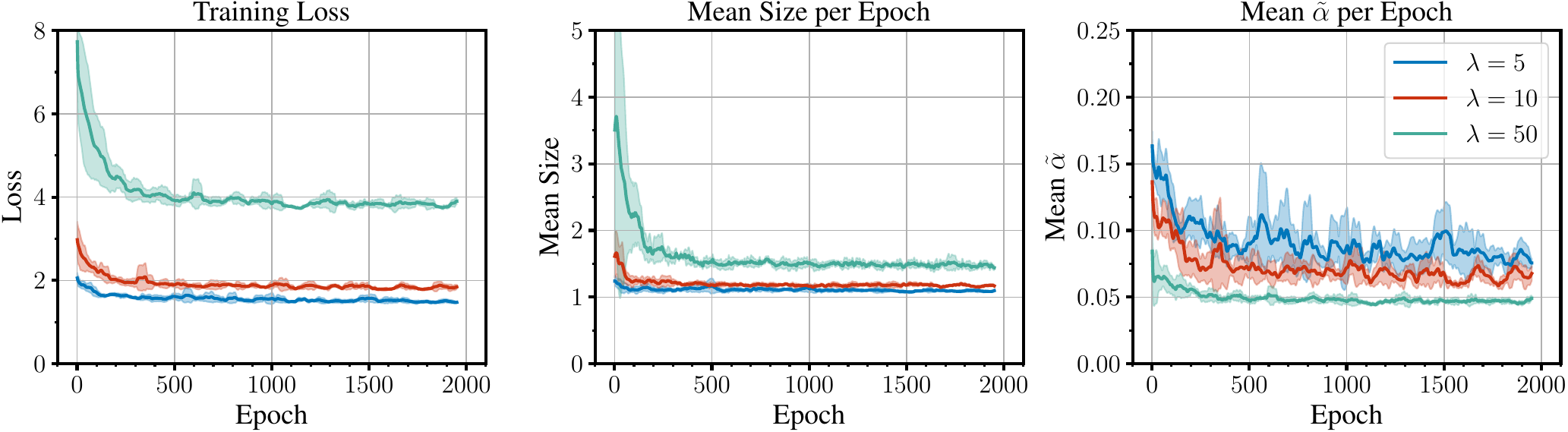}
    \caption{Training curves for $\lambda \in \{5,10,50\}$, averaged over 5 runs and smoothed with a moving average of size 50 for clarity. Shaded regions show $\pm1$ standard deviation across runs. \textbf{Left:} training loss. \textbf{Center:} mean set size. \textbf{Right:} mean adaptive miscoverage $\tilde\alpha$.}
    \label{fig:training}
\end{figure*}

\begin{proposition}[Monotonicity of leave-one-out size under constant $\alpha$]
\label{prop:monotone-size}
Define
\[
\overline{\rm Size}_n(\lambda) := \frac{1}{n}\sum_{j=1}^n \rm Size\big(\hat{C}_{n-1}^{\alpha^*(\lambda)}(X_j)\big),
\]
where 
\[
\alpha^*(\lambda) := \underset{\alpha \in (0,1)}{\rm argmin} \ \frac{1}{n}\sum_{j=1}^n \rm Size(\hat{C}_{n-1}^{\alpha}(X_j)) + \lambda \alpha,
\]
and assume that this minimizer exists.  

Assume moreover that the size function $\rm Size(\hat{C}_{n-1}^{\alpha}(X_j))$ is non-increasing in $\alpha$, which holds for any conformal set constructed using an e-value and threshold $1/\alpha$, in particular for the conformal set defined in (\ref{conformal_set}).  

Then $\overline{\rm Size}_n(\lambda)$ is non-decreasing in $\lambda$.
\end{proposition}

\section{EXPERIMENTAL EVALUATION}

To demonstrate the effectiveness of our approach, we conduct experiments on the  CIFAR-10 dataset \citep{krizhevsky2009learning}, a standard benchmark in computer vision consisting of 60,000 $32 \times 32$ color images evenly distributed across 10 object classes (such as airplanes, cats, and trucks). The dataset is split into 50,000 training examples and 10,000 test examples.

We train a deep neural network, denoted by $f$, on the full CIFAR-10 training set and treat it as a black-box predictor throughout our experiments. Specifically, we use an EfficientNet-B0 model \citep{tan2019efficientnet} trained to minimize the cross-entropy loss using stochastic gradient descent (SGD) with momentum 0.9, a learning rate of 0.1, weight decay of $5 \times 10^{-4}$, and cosine annealing over 100 epochs. We use a batch size of 512 and apply standard data augmentation techniques during training. At the end of training, the model $f$ achieves a training accuracy of 98.6\% and a test accuracy of~91.1\%.

For our experiments, we choose the cross-entropy as the score function:
\[
S(x,y)=-\log p_f(y|x),
\]
where $p_f(y|x)$ is the probability that the pretrained model $f$ assigns to label $y$ for a given input image $x$.

To construct the data needed for training an adaptive coverage policy, we randomly split the CIFAR-10 test set into a calibration set and a remaining set from which we randomly sample a test point. In our experiments, we fix the calibration set size to $n = 100$.

For the neural network $\tilde{\alpha}_\theta$, we use a simple architecture consisting of a fully connected feedforward network with one hidden layer of 32 units and ReLU activation. The output layer has a single neuron followed by a sigmoid activation to produce outputs between 0 and~1.

We optimize the loss function (\ref{loss}) using sigmoid smoothing (see Remark~\ref{rk:classif}) with $k=100$, and the Adam optimizer \citep{kingma2015adam} with a learning rate of~$1\times 10^{-3}$. Training is performed with a batch size of~64 over 2000 epochs.

For a calibration set sampled uniformly at random, Figure~\ref{fig:training} illustrates the training dynamics of the coverage policy produced by Algorithm \ref{algorithm} under different choices of the regularization strength $\lambda \in \{5, 10, 50\}$. The training curves decrease and converge, indicating that the model is optimizing effectively. Moreover, both the prediction set size and the expected miscoverage stabilize asymptotically as training progresses. This effective convergence highlights the value of the leave-one-out procedure, which provides sufficient signal for successfully training a coverage policy.

\begin{figure}[h!]
    \centering
    \includegraphics[width=.45\textwidth]{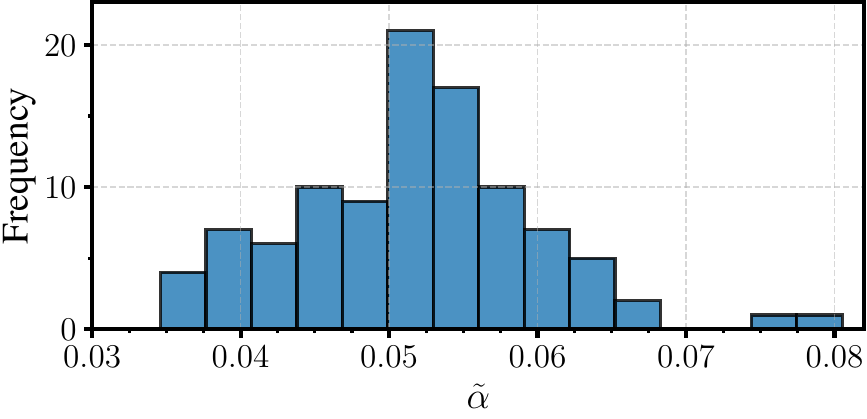}
    \caption{Distribution of adaptive miscoverage $\tilde\alpha$ across 100 randomly sampled test points.}
    \label{fig:hist_alpha}
\end{figure}

Once training is complete, we can evaluate the model at test time. We begin by illustrating the results with a model $\tilde{\alpha}_\theta$ trained using $\lambda = 50$. The model is evaluated on 100 test points sampled uniformly at random from the portion of the test set that is not used for calibration. Figure~\ref{fig:hist_alpha} displays the resulting distribution of adaptive miscoverage levels $\tilde{\alpha}$. The $\tilde{\alpha}$ values are reasonably spread, as would be expected if the coverage policy adapts to the varying difficulty of different predictions. 

Figure~\ref{fig:hist_set_sizes} shows the conformal sets produced by our method, henceforth referred to as \methodname{e-adaptive}, with an average set size of 1.30. We compare this to two baselines. The first, \methodname{e-fixed}, also uses e-values but constructs conformal sets (\ref{conformal_set}) with a fixed $\alpha$ equal to the empirical mean $\mathbb{E}[\tilde{\alpha}]$ computed from the 100 test points under our adaptive method. This yields a larger average set size of 1.42, highlighting the efficiency gains of adapting~$\alpha$ to individual samples. The second baseline, \methodname{p-fixed}, employs standard conformal prediction with p-values and the same fixed $\alpha$. While it achieves the smallest average set size of 1.01, it provides no principled way to select $\alpha$ in practice, making the resulting set sizes unpredictable. By contrast, both \methodname{e-adaptive} and \methodname{e-fixed} naturally support post-hoc selection of $\alpha$, combining flexibility with valid coverage guarantees.

\begin{figure}[h!]
    \centering
    \includegraphics[width=.45\textwidth]{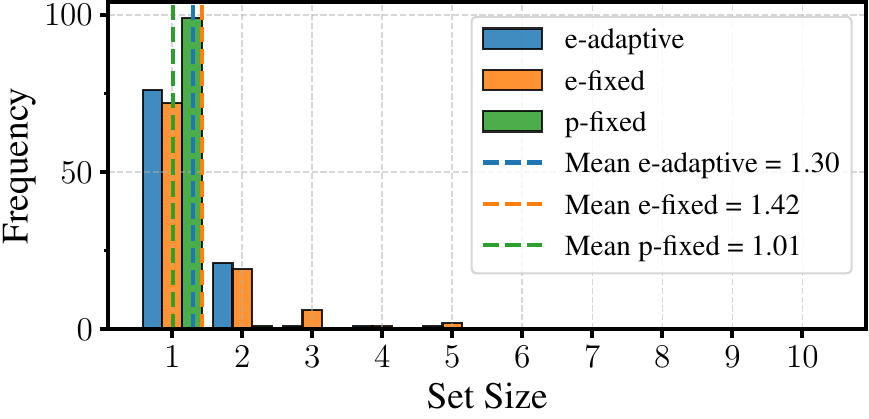}
    \caption{Distribution of conformal set sizes across 100 test points for the three methods.}
    \label{fig:hist_set_sizes}
\end{figure}

To provide a clearer picture of the benefits of our \methodname{e-adaptive} method, we aggregate results over five independent calibration sets. For each calibration set, we apply Algorithm \ref{algorithm} to train a network $\tilde{\alpha}_\theta$ and construct the associated conformal sets for different values of $\lambda$. We then compute the average set size over 100 randomly chosen test points. We report the mean and standard deviation of these average set sizes across the five runs, and compare against the two baselines introduced above (\methodname{e-fixed} and \methodname{p-fixed}) in Table \ref{tab:results}.

To select the regularization strength in practice, we can apply Algorithm \ref{algorithm_lambda}, which iteratively adjusts $\lambda$ to achieve some target mean prediction set size $M$. Figure~\ref{fig:lambda_selection} illustrates this process for an initial $\lambda=40$, target $M=2$, and tolerance $\varepsilon=0.1$. The bracketing phase completes in a single iteration, as the mean size is below $M$ for $\lambda=40$ and above $M$ for $\lambda=80$. The subsequent bisection phase then refines $\lambda$ through $60$, $70$, and $65$, converging smoothly to a value that meets the target within the prescribed tolerance.

In this example, Algorithm~\ref{algorithm_lambda} converges to a final mean set size of 2.00. We compare this value to the expected test-time prediction set, computed over 100 randomly sampled test points, yielding 2.07, which is very close to the final mean. This suggests that Theorem \ref{thm:loo-size-consistency}, proven in the idealized constant-output case, extends in practice to neural networks with input-dependent outputs. Moreover, the middle plot from Figure~\ref{fig:training} shows that the monotonicity of the set size with respect to $\lambda$, established in Proposition~\ref{prop:monotone-size} for the constant case, also holds for these more complex networks. Intuitively, this is expected because the network’s $\tilde\alpha$ outputs vary smoothly with the inputs in practice. Moreover, the calibration sum changes little across data points, as it concentrates around $n$ times the mean expected score. Figure~\ref{fig:hist_alpha} confirms that the distribution of network outputs remains smooth as test features vary. Together, these observations indicate that the theoretical properties derived for constant-output networks provide reliable guidance in the input-dependent setting.

To further support the theoretical guarantees in Theorem \ref{thm:loo-size-consistency}, we present an additional regression experiment in Appendix \ref{appendix:reg}.

\begin{table*}[h]
\centering
\caption{Prediction set size comparison across methods and regularization values.}
\begin{tabular}{lccc||c}
\toprule
Method & $\lambda=5$ & $\lambda=10$ & $\lambda=50$ & Post-hoc selection of $\alpha$? \\
\midrule
\methodname{e-adaptive} & 1.21$\pm$0.10 & 1.23$\pm$0.12 & 1.63$\pm$0.19 & \cmark \\
\methodname{e-fixed}    & 1.26$\pm$0.12 & 1.33$\pm$0.19 & 1.92$\pm$0.29 & \cmark \\
\methodname{p-fixed}    & 0.96$\pm$0.05 & 0.98$\pm$0.05 & 1.18$\pm$0.12 & \xmark \\
\bottomrule
\end{tabular}
\label{tab:results}
\end{table*}

\begin{figure}[h!]
    \centering
    \includegraphics[width=.45\textwidth]{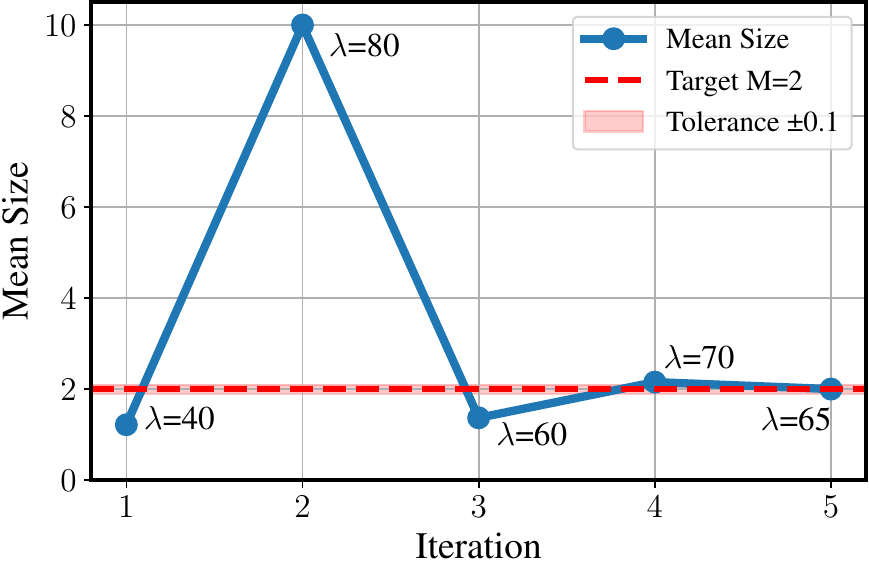}
    \caption{Evolution of $\lambda$ and mean set size during Algorithm~\ref{algorithm_lambda}, showing the bracketing and bisection phases converging to the target $M=2$.}
    \label{fig:lambda_selection}
\end{figure}

\section{CONCLUSION}

We have introduced an extension of conformal prediction that allows the miscoverage level to be set adaptively. By leveraging e-values and their post-hoc validity, our approach tailors the miscoverage level to the difficulty of each test sample via a coverage policy trained on the calibration set using a leave-one-out procedure. The method is theoretically grounded, providing adaptive miscoverage at deployment while maintaining valid marginal coverage. Unlike standard conformal methods, which offer no principled way to choose $\alpha$ and provide no insight into expected set size, our approach both adapts coverage to the data and enables optimizing the coverage policy for a desired expected test-time set size.

Note that when a history of conformal prediction episodes is available, the coverage policy can be trained directly on them, making leave-one-out unnecessary. However, in such cases, it is also possible to anticipate the optimal coverage level statistically and apply standard conformal prediction methods with p-values, which may be preferable. The strength of our method lies in the fact that the leave-one-out approach allows us to emulate pseudo conformal prediction episodes to train a coverage policy, without requiring access to any prior history of conformal predictions.


\subsubsection*{Acknowledgements}
The authors thank the anonymous reviewers for their helpful feedback that improved this work.

Funded by the European Union (ERC-2022-SYG-OCEAN-101071601).
Views and opinions expressed are however those of the author(s) only and do not
necessarily reflect those of the European Union or the European Research Council
Executive Agency. Neither the European Union nor the granting authority can be
held responsible for them. 

This publication is part of the Chair ``Markets and Learning,'' supported by Air Liquide, BNP PARIBAS ASSET MANAGEMENT Europe, EDF, Orange and SNCF, sponsors of the Inria Foundation.

This work has also received support from the French government, managed by the National Research Agency, under the France 2030 program with the reference ``PR[AI]RIE-PSAI" (ANR-23-IACL-0008).

\bibliographystyle{plainnat} 
\bibliography{references}

@article{gauthier2025evaluesexpandscopeconformal,
   title={E-Values Expand the Scope of Conformal Prediction},
   author={Etienne Gauthier and Francis Bach and Michael I. Jordan},
   journal={arXiv preprint arXiv:2503.13050},
   year={2025}
}

@book{vovk2005algorithmiclearning,
author = {Vovk, Vladimir and Gammerman, Alex and Shafer, Glenn},
title = {Algorithmic Learning in a Random World},
publisher = {Springer-Verlag},
year = {2005}
}

@inproceedings{papadopoulos2002inductivecm,
  title={Inductive Confidence Machines for Regression},
  author={Harris Papadopoulos and Kostas Proedrou and Vladimir Vovk and Alexander Gammerman},
  booktitle={European Conference on Machine Learning},
  year={2002}
}

@article{angelopoulos2023gentle,
author = {Angelopoulos, Anastasios N. and Bates, Stephen},
title = {Conformal Prediction: A Gentle Introduction},
year = {2023},
issue_date = {Mar 2023},
publisher = {Now Publishers Inc.},
address = {Hanover, MA, USA},
volume = {16},
number = {4},
journal = {Foundations and Trends in Machine Learning},
pages = {494–591},
numpages = {114}
}

@article{angelopoulos2024theoreticalfoundationsconformalprediction,
   title={Theoretical Foundations of Conformal Prediction},
   author={Anastasios N. Angelopoulos and Rina Foygel Barber and Stephen Bates},
   journal={arXiv preprint arXiv:2411.11824},
   year={2024}
}

@article{Vovk2025conformaleprediction,
   title={Conformal e-prediction},
   volume={166},
   journal={Pattern Recognition},
   publisher={Elsevier BV},
   author={Vovk, Vladimir},
   year={2025},
 pages={111674} }

@InProceedings{balinsky2024EnhancingCP,  title =  {Enhancing Conformal Prediction Using E-Test Statistics},  author =       {Balinsky, Alexander A. and Balinsky, Alexander David},  booktitle =  {Symposium on Conformal and Probabilistic Prediction with Applications},  year =  {2024}}

@Article{wang2022fdr,
  author={Ruodu Wang and Aaditya Ramdas},
  title={{False discovery rate control with e‐values}},
  journal={Journal of the Royal Statistical Society Series B},
  year=2022,
  volume={84},
  number={3},
  pages={822-852}
}

@article{xu2024postselectioninference,
author = {Ziyu Xu and Ruodu Wang and Aaditya Ramdas},
title = {{Post-selection inference for e-value based confidence intervals}},
volume = {18},
journal = {Electronic Journal of Statistics},
number = {1},
publisher = {Institute of Mathematical Statistics and Bernoulli Society},
pages = {2292 -- 2338},
year = {2024}
}

@article{grunwald2024beyond,
author = {Peter Grünwald },
title = {Beyond {N}eyman-{P}earson: E-values enable hypothesis testing with a data-driven alpha},
journal = {Proceedings of the National Academy of Sciences},
volume = {121},
number = {39},
pages = {e2302098121},
year = {2024}}

@article{ramdas2024hypothesistestingevalues,
  title={Hypothesis testing with e-values},
  author={Ramdas, Aaditya and Wang, Ruodu},
  journal={Foundations and Trends{\textregistered} in Statistics},
  volume={1},
  number={1-2},
  pages={1--390},
  year={2025}
  }

@article{koning2025measuring,
   title={Measuring Evidence against Exchangeability and Group Invariance with E-values},
   author={Nick W. Koning},
   journal={arXiv preprint arXiv:2310.01153},
   year={2025}
}

@article{koning2024posthocalphahypothesistesting,
   title={Post-hoc $\alpha$ Hypothesis Testing and the Post-hoc p-value},
   author={Nick W. Koning},
   journal={arXiv preprint arXiv:2312.08040},
   year={2024}
}

@techreport{krizhevsky2009learning,
  title        = {Learning Multiple Layers of Features from Tiny Images},
  author       = {Krizhevsky, Alex},
  year         = {2009},
  institution  = {University of Toronto},
  type         = {Technical Report}
}

@InProceedings{tan2019efficientnet,
  title = 	 {{E}fficient{N}et: Rethinking Model Scaling for Convolutional Neural Networks},
  author =       {Tan, Mingxing and Le, Quoc},
  booktitle = 	 {International Conference on Machine Learning},
  year = 	 {2019}
}

@article{hoeffding1963,
  author = {Hoeffding, Wassily},
  journal = {Journal of the American Statistical Association},
  number = 301,
  pages = {13-30},
  title = {Probability Inequalities for Sums of Bounded Random Variables},
  volume = 58,
  year = 1963
}

@article{vovk2021evalues,
  title={E-values: Calibration, combination and applications},
  author={Vovk, Vladimir and Wang, Ruodu},
  journal={The Annals of Statistics},
  volume={49},
  number={3},
  pages={1736--1754},
  year={2021},
  publisher={Institute of Mathematical Statistics}
}

@inproceedings{
angelopoulos2021uncertainty,
title={Uncertainty Sets for Image Classifiers using Conformal Prediction},
author={Anastasios N. Angelopoulos and Stephen Bates and Michael I. Jordan and Jitendra Malik},
booktitle={International Conference on Learning Representations},
year={2021}
}

@inproceedings{
kingma2015adam,
title={Adam: A method for stochastic optimization},
author={Diederik P. Kingma and Jimmy Ba},
booktitle={International Conference on Learning Representations},
year={2015}
}

@inproceedings{cherian2024llmvalidity,
    author = {Cherian, John and Gibbs, Isaac and Cand\`es, Emmanuel},
    title = "Large language model validity via enhanced conformal prediction methods",
    booktitle = "Advances in Neural Information Processing Systems",
    year = 2024 
}

@book{vaart1998asymptotics, place={Cambridge}, title={Asymptotic Statistics}, publisher={Cambridge University Press}, author={van der Vaart, Aad W.}, year={1998}, collection={Cambridge Series in Statistical and Probabilistic Mathematics}}

@article{lei2018inference,
author = {Jing Lei and Max G’Sell and Alessandro Rinaldo and Ryan J. Tibshirani and Larry Wasserman},
title = {Distribution-Free Predictive Inference for Regression},
journal = {Journal of the American Statistical Association},
volume = {113},
number = {523},
pages = {1094--1111},
year = {2018},
publisher = {ASA Website}
}

@book{balasubramanian2014conformal,
author = {Balasubramanian, Vineeth and Ho, Shen-Shyang and Vovk, Vladimir},
title = {Conformal Prediction for Reliable Machine Learning: Theory, Adaptations and Applications},
year = {2014}
}

@inproceedings{su2024llm,
  title={{API} Is Enough: Conformal Prediction for Large Language Models Without Logit-Access},
  author={Jiayuan Su and Jing Luo and Hongwei Wang and Lu Cheng},
  booktitle={Conference on Empirical Methods in Natural Language Processing},
  year={2024}
}

@article{laxhammar2015inductive,
author = {Laxhammar, Rikard and Falkman, G\"{o}ran},
title = {Inductive conformal anomaly detection for sequential detection of anomalous sub-trajectories},
year = {2015},
volume = {74},
number = {1–2},
journal = {Annals of Mathematics and Artificial Intelligence},
pages = {67–94}
}

@article{bates2023testing,
author = {Bates, Stephen and Candès, Emmanuel and Lei, Lihua and Romano, Yaniv and Sesia, Matteo},
year = {2023},
pages = {},
title = {Testing for outliers with conformal p-values},
volume = {51},
number={1},
journal = {The Annals of Statistics}
}

@inproceedings{fisch2021few,
  title={Few-shot conformal prediction with auxiliary tasks},
  author={Fisch, Adam and Schuster, Tal and Jaakkola, Tommi and Barzilay, Regina},
  booktitle={International Conference on Machine Learning},
  year={2021}
}

@InProceedings{chernozhukov2018exact,
  title = 	 {Exact and Robust Conformal Inference Methods for Predictive Machine Learning with Dependent Data},
  author =       {Chernozhukov, Victor and W\"{u}thrich, Kaspar and Yinchu, Zhu},
  booktitle = 	 {Conference On Learning Theory},
  year = 	 {2018}
}

@InProceedings{johnstone2021conformal,  title =  {Conformal uncertainty sets for robust optimization},  author =       {Johnstone, Chancellor and Cox, Bruce},  booktitle =  {Proceedings of the Symposium on Conformal and Probabilistic Prediction and Applications},  pages =  {72--90},  year =  {2021}, volume =  {152}}

@InProceedings{cella2021valid,
  title = 	 {Valid Inferential Models for Prediction in Supervised Learning Problems},
  author =       {Cella, Leonardo and Martin, Ryan},
  booktitle = 	 {Proceedings of the International Symposium on Imprecise Probability: Theories and Applications},
  pages = 	 {72--82},
  year = 	 {2021},
  volume = 	 {147}
}

@article{simmons2011false,
author = {Joseph P. Simmons and Leif D. Nelson and Uri Simonsohn},
title ={False-Positive Psychology: Undisclosed Flexibility in Data Collection and Analysis Allows Presenting Anything as Significant},
journal = {Psychological Science},
volume = {22},
number = {11},
pages = {1359-1366},
year = {2011}
}

@article{head2015phacking,
    author = {Head, Megan L. AND Holman, Luke AND Lanfear, Rob AND Kahn, Andrew T. AND Jennions, Michael D.},
    journal = {PLOS Biology},
    publisher = {Public Library of Science},
    title = {The Extent and Consequences of P-Hacking in Science},
    year = {2015},
    month = {03},
    volume = {13},
    pages = {1-15},
    number = {3}
}

@article{grunwald2024safetesting,
  title={Safe testing},
  author={Gr{\"u}nwald, Peter and de Heide, Rianne and Koolen, Wouter},
  journal={Journal of the Royal Statistical Society Series B: Statistical Methodology},
  volume={86},
  number={5},
  pages={1091--1128},
  year={2024},
  publisher={Oxford University Press UK}
}

@book{shafer2019gametheoretic,
title = "Game-Theoretic Foundations for Probability and Finance",
author = "Glenn Shafer and Vladimir Vovk",
year = "2019",
publisher = "Wiley"
}

@article{chugg2025admissibility,
title = {On admissibility in post-hoc hypothesis testing},
journal = {International Journal of Approximate Reasoning},
volume = {191},
pages = {109634},
year = {2026},
author = {Ben Chugg and Tyron Lardy and Aaditya Ramdas and Peter Grünwald},

}

@inproceedings{gauthier2025backwardconformalprediction,
author = {Etienne Gauthier and Francis Bach and Michael I. Jordan},
title = {Backward Conformal Prediction},
year = {2025},
booktitle = {Advances in Neural Information Processing Systems}
}

@article{liu2026stbcptighteningcoveragebound,title={ST-BCP: Tightening Coverage Bound for Backward Conformal Prediction via Non-Conformity Score Transformation}, 
      author={Junxian Liu and Hao Zeng and Hongxin Wei},
  journal={arXiv preprint arXiv:2602.01733},
      year={2026}
}

@inproceedings{gammerman1998learning,
author = {Gammerman, Alex and Vovk, Volodya and Vapnik, Vladimir},
title = {Learning by transduction},
year = {1998},
booktitle = {Conference on Uncertainty in Artificial Intelligence}
}

@article{shafer2008tutorial,
author = {Shafer, Glenn and Vovk, Vladimir},
title = {A Tutorial on Conformal Prediction},
year = {2008},
volume = {9},
journal = {Journal of Machine Learning Research},
pages = {371–421},
}

@inproceedings{berthet2020perturbed,
author = {Berthet, Quentin and Blondel, Mathieu and Teboul, Olivier and Cuturi, Marco and Vert, Jean-Philippe and Bach, Francis},
title = {Learning with differentiable perturbed optimizers},
year = {2020},
booktitle = {Advances in Neural Information Processing Systems}
}

\section*{Checklist}

\begin{enumerate}

  \item For all models and algorithms presented, check if you include:
  \begin{enumerate}
    \item A clear description of the mathematical setting, assumptions, algorithm, and/or model. [Yes]
    \item An analysis of the properties and complexity (time, space, sample size) of any algorithm. [Yes]
    \item (Optional) Anonymized source code, with specification of all dependencies, including external libraries. [No] We will publicly release the code.
  \end{enumerate}

  \item For any theoretical claim, check if you include:
  \begin{enumerate}
    \item Statements of the full set of assumptions of all theoretical results. [Yes]
    \item Complete proofs of all theoretical results. [Yes]
    \item Clear explanations of any assumptions. [Yes]     
  \end{enumerate}

  \item For all figures and tables that present empirical results, check if you include:
  \begin{enumerate}
    \item The code, data, and instructions needed to reproduce the main experimental results (either in the supplemental material or as a URL). [Yes]
    \item All the training details (e.g., data splits, hyperparameters, how they were chosen). [Yes]
    \item A clear definition of the specific measure or statistics and error bars (e.g., with respect to the random seed after running experiments multiple times). [Yes]
    \item A description of the computing infrastructure used. (e.g., type of GPUs, internal cluster, or cloud provider). [Yes]
  \end{enumerate}

  \item If you are using existing assets (e.g., code, data, models) or curating/releasing new assets, check if you include:
  \begin{enumerate}
    \item Citations of the creator If your work uses existing assets. [Yes]
    \item The license information of the assets, if applicable. [Yes]
    \item New assets either in the supplemental material or as a URL, if applicable. [Yes]
    \item Information about consent from data providers/curators. [Yes]
    \item Discussion of sensible content if applicable, e.g., personally identifiable information or offensive content. [Not Applicable]
  \end{enumerate}

  \item If you used crowdsourcing or conducted research with human subjects, check if you include:
  \begin{enumerate}
    \item The full text of instructions given to participants and screenshots. [Not Applicable]
    \item Descriptions of potential participant risks, with links to Institutional Review Board (IRB) approvals if applicable. [Not Applicable]
    \item The estimated hourly wage paid to participants and the total amount spent on participant compensation. [Not Applicable]
  \end{enumerate}

\end{enumerate}

\clearpage
\appendix
\thispagestyle{empty}

\onecolumn
\aistatstitle{Supplementary Materials}

\section{PROOF OF THEOREM \ref{thm:loo-size-consistency}}
\label{appendix:proof}

\begin{theorem}[Consistency of leave-one-out size]
Assume that the calibration samples $(X_i,Y_i)$ are i.i.d., and let $\alpha \in (0,1)$ be a given target miscoverage level.  

Assume one of the following two cases holds:
\begin{enumerate}[label=(\roman*)]
\item \textbf{(Classification, sigmoid smoothing)}  
The size is given by the smooth sigmoid approximation defined in (\ref{eq:smooth_sigmoid}) with some parameter $k>0$.  
The score function $S$ is bounded and takes values in $[S_{\min},S_{\max}]$ with $0<S_{\min}\le S_{\max}<\infty$, and $n > S_{\max}/S_{\min}$.  

\item \textbf{(Regression, MAE score)}  
The size is defined in (\ref{eq:size_reg}).  
The score function $S$ is bounded and takes values in $[0,S_{\max}]$ with $S_{\max}<\infty$, and the miscoverage level satisfies $\alpha > 1/n$.  
\end{enumerate}

Let \(\overline{\rm Size}_n := \frac{1}{n}\sum_{j=1}^n \rm Size(\hat{C}_{n-1}^{\alpha}(X_j))\) denote average size under the leave-one-out protocol. Then, under either (i) or (ii),  
\[
\Big|\overline{\rm Size}_n - \mathbb{E}\left[\rm Size\left(\hat{C}_n^{\alpha}(X_{\rm test})\right)\right]\Big| = O_P\!\left(\frac{1}{\sqrt{n}}\right),
\]
i.e., the average size consistently estimates the expected test-time size at rate \(1/\sqrt{n}\) in probability.
\end{theorem}

\begin{proof}
Let $\mu := \mathbb{E}[S(X,Y)]$ denote the expected score.  
We focus first on case (ii) (regression with MAE score). We want to compare the leave-one-out average size
\(\overline{\rm Size}_n := \frac{1}{n} \sum_{j=1}^n \rm Size(\hat{C}_{n-1}^{\alpha}(X_j))\)
with the expected test-time size 
\(\mathbb{E}[\rm Size(\hat{C}_n^{\alpha}(X_{\rm test}))]\). For the MAE score, these quantities have the explicit forms
\[
\overline{\rm Size}_n = \frac{1}{n} \sum_{j=1}^n \frac{2 \sum_{i \neq j} S(X_i,Y_i)}{n\alpha - 1}, 
\qquad 
\mathbb{E}[\rm Size(\hat{C}_n^{\alpha}(X_{\rm test}))] = \frac{2 n \mu}{(n+1)\alpha - 1}.
\]
Subtracting and rearranging terms gives
\begin{align*}
\Big|\overline{\rm Size}_n - \mathbb{E}[\rm Size(\hat{C}_n^{\alpha}(X_{\rm test}))]\Big|
&= \Bigg| \frac{2(n-1)}{n\alpha - 1} \cdot \frac{1}{n} \sum_{i=1}^n S(X_i,Y_i) - \frac{2 n \mu}{(n+1)\alpha - 1} \Bigg| \\
&\le \underbrace{\frac{2(n-1)}{n\alpha - 1} \Bigg|\frac{1}{n} \sum_{i=1}^n S(X_i,Y_i) - \mu \Bigg|}_{\text{fluctuation term}} 
+ \underbrace{\Bigg| \frac{2(n-1)\mu}{n\alpha - 1} - \frac{2 n \mu}{(n+1)\alpha - 1} \Bigg|}_{\text{bias term}}.
\end{align*}
The first term is a standard concentration term. Since the $S(X_i,Y_i)$ are i.i.d.\ and bounded by assumption, Hoeffding's inequality \citep{hoeffding1963} gives, for any $\delta>0$,  
\[
\Bigg|\frac{1}{n} \sum_{i=1}^n S(X_i,Y_i) - \mu\Bigg| \le S_{\max}\sqrt{\frac{\log(2/\delta)}{2n}} \quad \text{with probability at least } 1-\delta.
\]
The second term is purely deterministic and can be verified to satisfy
\[
\Bigg| \frac{2(n-1)\mu}{n\alpha - 1} - \frac{2 n \mu}{(n+1)\alpha - 1} \Bigg| = O\Big(\frac{1}{n}\Big).
\]
Combining these two bounds, we see that
\[
\Big|\overline{\rm Size}_n - \mathbb{E}[\rm Size(\hat{C}_n^{\alpha}(X_{\rm test}))]\Big| = O_P\Big(\frac{1}{\sqrt{n}}\Big),
\]
as claimed.

We now turn to case (i). The argument follows the strategy of \citet[Theorem 3.1]{gauthier2025backwardconformalprediction}. We use the same notations: for each calibration index $j=1,\dots,n$, define the random vector
\[
\mathbf{E}^j := \left(\frac{S(X_{j},y)}{\tfrac{1}{n}\big(\sum_{i \neq j} S(X_i,Y_i)+S(X_{j},y)\big)} \right)_{y \in \mathcal{Y}},
\]
and its population counterpart
\[
\tilde{\mathbf{E}}^j := \left(\frac{S(X_{j},y)}{\mu}\right)_{y \in \mathcal{Y}}.
\]
Similarly, for the test point we write
\[
\mathbf{E}^{\rm test} := \left(\frac{S(X_{\rm test},y)}{\tfrac{1}{n+1}\big(\sum_{i=1}^{n} S(X_i,Y_i)+S(X_{\rm test},y)\big)} \right)_{y \in \mathcal{Y}},
\qquad 
\tilde{\mathbf{E}}^{\rm test} := \left(\frac{S(X_{\rm test},y)}{\mu}\right)_{y \in \mathcal{Y}}.
\]
Define the function
\[
f : \mathbb{R}_+^{|\mathcal{Y}|} \to \mathbb{R}, 
\qquad 
(\mathbf{E}_y)_{y \in \mathcal{Y}} \mapsto \sum_{y \in \mathcal{Y}} \sigma\!\left(k\left(\frac{1}{\alpha}-\mathbf{E}_y\right)\right),
\]
so that the size can be written as $f(\mathbf{E}^j)$ in the leave-one-out case, and as $f(\mathbf{E}^{\rm test})$ at test time. First, observe that $f$ is Lipschitz-continuous for the $\ell_\infty$ norm. Indeed, for $\mathbf{E}^1, \mathbf{E}^2 \in \mathbb{R}_+^{|\mathcal{Y}|}$ we have
\[
\big|f(\mathbf{E}^1)-f(\mathbf{E}^2)\big|
= \Big|\sum_{y\in\mathcal Y}\big(\sigma(k(1/\alpha-\mathbf{E}^1_y))-\sigma(k(1/\alpha-\mathbf{E}^2_y))\big)\Big|
\le \sum_{y\in\mathcal Y}\big|\sigma(k(1/\alpha-\mathbf{E}^1_y))-\sigma(k(1/\alpha-\mathbf{E}^1_y))\big|.
\]
The slope of the sigmoid is uniformly bounded by $1/4$, hence $\sigma$ is $1/4$-Lipschitz. Therefore:
\[
|f(\mathbf{E}^1)-f(\mathbf{E}^2)|
\le \frac{k}{4}\sum_{y\in\mathcal Y}|\mathbf{E}^1_y-\mathbf{E}^2_y|
\le \frac{k|\mathcal{Y}|}{4}\|\mathbf{E}^1-\mathbf{E}^2\|_\infty.
\]
Thus $f$ is Lipschitz-continuous for the $\ell_\infty$-norm with Lipschitz constant $L:=k|\mathcal Y|/4$.

Now, we write
\begin{align*}
\overline{\rm Size}_n - \mathbb{E}[\rm Size(\hat{C}_n^{\alpha}(X_{\rm test}))]
&= \frac{1}{n}\sum_{j=1}^n f(\mathbf{E}^j) - \mathbb{E}[f(\mathbf{E}^{\rm test})] \\
&= \underbrace{\frac{1}{n}\sum_{j=1}^n \big(f(\mathbf{E}^j) - f(\tilde{\mathbf{E}}^j)\big)}_{=: \, T_1}
 + \underbrace{\frac{1}{n}\sum_{j=1}^n \big(f(\tilde{\mathbf{E}}^j) - \mathbb{E}[f(\tilde{\mathbf{E}}^{\rm test})]\big)}_{=: \, T_2}
 + \underbrace{\mathbb{E}[f(\tilde{\mathbf{E}}^{\rm test})] - \mathbb{E}[f(\mathbf{E}^{\rm test})]}_{=: \, T_3}.
\end{align*}
We now bound each term separately. Using the Lipschitz continuity of $f$ with constant $L$, we have
\begin{align*}
|T_1| &\le \frac{1}{n} \sum_{j=1}^n \left| f(\mathbf{E}^j) - f(\tilde{\mathbf{E}}^j) \right| 
\le \frac{L}{n} \sum_{j=1}^n \left\| \mathbf{E}^j - \tilde{\mathbf{E}}^j \right\|_\infty
\le L S_{\max} \frac{\sqrt{\frac{\log(2/\delta)}{2n}} + \frac{2}{n}}{\mu \left(S_{\min}/S_{\max}-\frac{1}{n}\right)}
\quad \text{with probability $\ge 1-\delta$},
\end{align*}
for all $\delta > 0$, where the last inequality follows from \citet{gauthier2025backwardconformalprediction}.
By Hoeffding's inequality,
\[
|T_2| \le \sqrt{\frac{\log(2/\delta)}{2n}} \quad \text{with probability $\ge 1-\delta$}.
\]
Finally, for the bias term $T_3$, we have
\begin{align*}
|T_3| &\le \mathbb{E}\left[ \left| f(\tilde{\mathbf{E}}^{\rm test}) - f(\mathbf{E}^{\rm test}) \right| \right]
\le L \, \mathbb{E}\left[ \left\| \tilde{\mathbf{E}}^{\rm test} - \mathbf{E}^{\rm test} \right\|_\infty \right]
\le \frac{2 L S_{\max}^2}{\mu S_{\min}} \frac{n+1}{n} \sqrt{\frac{\pi}{2(n+1)}},
\end{align*}
again using \citet{gauthier2025backwardconformalprediction}.
Combining the three terms with an union bound, we conclude that with probability $\ge 1-\delta$,
\[
\Big|\overline{\rm Size}_n - \mathbb{E}[\rm Size(\hat{C}_n^{\alpha}(X_{\rm test}))]\Big| \le L S_{\max} \frac{\sqrt{\frac{\log(4/\delta)}{2n}} + \frac{2}{n}}{\mu \left(S_{\min}/S_{\max}-\frac{1}{n}\right)} + \sqrt{\frac{\log(4/\delta)}{2n}} + \frac{2 L S_{\max}^2}{\mu S_{\min}} \frac{n+1}{n} \sqrt{\frac{\pi}{2(n+1)}}.
\]
This shows that
\[
\Big|\overline{\rm Size}_n - \mathbb{E}[\rm Size(\hat{C}_n^{\alpha}(X_{\rm test}))]\Big| = O_P\Big(\frac{1}{\sqrt{n}}\Big).
\]
\end{proof}

Our proof relies on standard concentration inequalities to control deviations of the leave-one-out scores from their expectation; in particular, we employ Hoeffding's inequality for its simplicity.  
In the regression case, we explicitly assume that $\alpha > 1/n$ so that the conformal set sizes are well-defined, and that the scores are bounded above by $S_{\max}$ to ensure bounded scores and enable the application of Hoeffding’s inequality. 
In the classification case, we additionally assume that the score is bounded below by a strictly positive value $S_{\min}>0$ to avoid division by zero when concentration inequalities are applied to denominators. 
Finally, the condition $n > S_{\max}/S_{\min}$ also ensures that denominators are strictly positive.

\section{PROOF OF PROPOSITION \ref{prop:monotone-size}}
\label{appendix:proof-monotone}

\begin{proposition}[Monotonicity of leave-one-out size under constant $\alpha$]
Define
\[
\overline{\rm Size}_n(\lambda) := \frac{1}{n}\sum_{j=1}^n \rm Size\big(\hat{C}_{n-1}^{\alpha^*(\lambda)}(X_j)\big),
\]
where 
\[
\alpha^*(\lambda) := \underset{\alpha \in (0,1)}{\rm argmin} \ \frac{1}{n}\sum_{j=1}^n \rm Size(\hat{C}_{n-1}^{\alpha}(X_j)) + \lambda \alpha,
\]
and assume that this minimizer exists.  

Assume moreover that the size function $\rm Size(\hat{C}_{n-1}^{\alpha}(X_j))$ is non-increasing in $\alpha$, which holds for any conformal set constructed using an e-value and threshold $1/\alpha$, in particular for the conformal set defined in (\ref{conformal_set}). 

Then $\overline{\rm Size}_n(\lambda)$ is non-decreasing in $\lambda$.
\end{proposition}
\begin{proof}
Let $0 < \lambda_1 < \lambda_2$, and denote the corresponding minimizers by $\alpha_1 := \alpha^*(\lambda_1)$ and $\alpha_2 := \alpha^*(\lambda_2)$. By optimality:
\[
\frac{1}{n}\sum_{j=1}^n \mathrm{Size}(\hat{C}_{n-1}^{\alpha_1}(X_j)) + \lambda_1 \alpha_1 
\le \frac{1}{n}\sum_{j=1}^n \mathrm{Size}(\hat{C}_{n-1}^{\alpha_2}(X_j)) + \lambda_1 \alpha_2,
\]
\[
\frac{1}{n}\sum_{j=1}^n \mathrm{Size}(\hat{C}_{n-1}^{\alpha_2}(X_j)) + \lambda_2 \alpha_2 
\le \frac{1}{n}\sum_{j=1}^n \mathrm{Size}(\hat{C}_{n-1}^{\alpha_1}(X_j)) + \lambda_2 \alpha_1.
\]
Adding these two inequalities yields
\[
(\lambda_2 - \lambda_1)(\alpha_2 - \alpha_1) \le 0 \quad \Rightarrow \quad \alpha_2 \le \alpha_1.
\]
Since $\rm Size(\hat{C}_{n-1}^{\alpha}(X_j))$ is non-increasing in $\alpha$, it follows that
\[
\overline{\rm Size}_n(\lambda_2) = \frac{1}{n}\sum_{j=1}^n \rm Size(\hat{C}_{n-1}^{\alpha_2}(X_j)) \ge 
\frac{1}{n}\sum_{j=1}^n \rm Size(\hat{C}_{n-1}^{\alpha_1}(X_j)) = \overline{\rm Size}_n(\lambda_1),
\]
which proves the claim.
\end{proof}

\section{ADDITIONAL EXPERIMENTS}
\label{appendix:reg}

We provide additional experiments in the regression setting to further demonstrate the empirical validity of our method and of Theorem \ref{thm:loo-size-consistency}. In this setup, we generate a synthetic regression dataset consisting of 100 training samples and 100 calibration samples. Each feature $X_i$ is drawn independently from a uniform distribution on~$[-5, 5]$, and the corresponding label is generated as
\[
Y_i = 2X_i + \varepsilon_i, \quad \varepsilon_i \sim \mathcal{N}(0,1).
\]
The predictor $f$ is taken to be a standard linear regression fit on the training data.

\begin{figure}[h!]
    \centering
    \includegraphics[width=.4\textwidth]{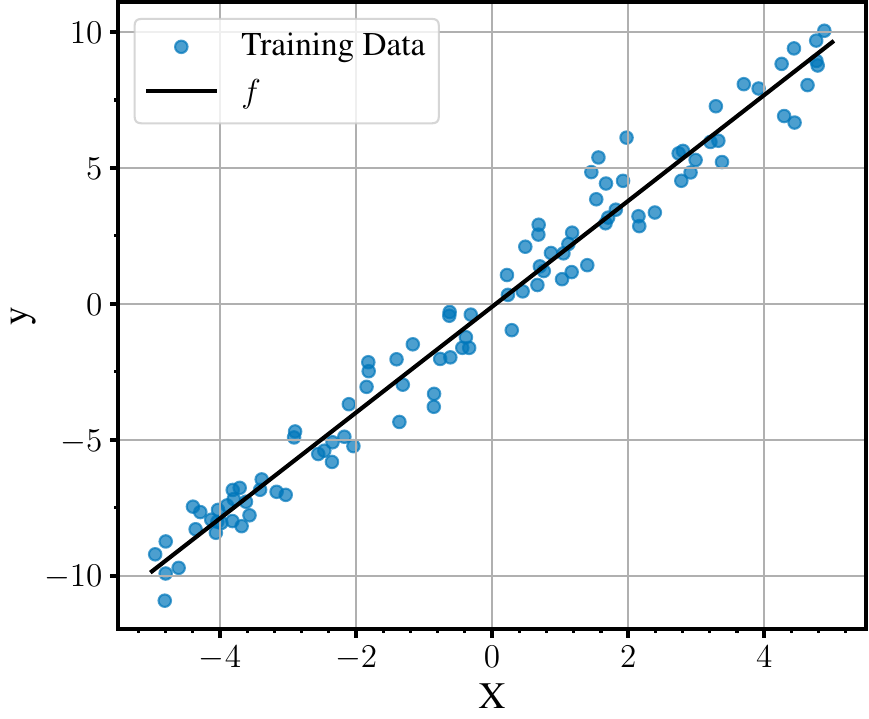}
    \caption{Visualization of the synthetic regression dataset.}
    \label{fig:reg_data}
\end{figure}

For this experiment, we reuse essentially the same neural network architecture for $\tilde{\alpha}_\theta$ as in the main experiment described in the paper: a fully connected feedforward network with one hidden layer of 32 units and ReLU activation, followed by a single output neuron with sigmoid activation. The main difference is that here we initialize the network deliberately so that the output is initially close to 1. This ensures that $\tilde{\alpha}_\theta$ does not fall below $1/n$, which could otherwise lead to ill-defined conformal set sizes.

Training is performed with the Adam optimizer at a learning rate of $10^{-3}$ and batch size 32 over 200 epochs. The network inputs are the leave-one-out score for the left-out calibration point and the sum of the remaining scores. We plot the training dynamics for $\lambda \in \{10, 20, 50\}$ in Figure~\ref{fig:reg_training}. The curves converge smoothly across all runs, demonstrating that the leave-one-out procedure allows the network to effectively train.

\begin{figure}[h!]
    \centering
    \includegraphics[width=\textwidth]{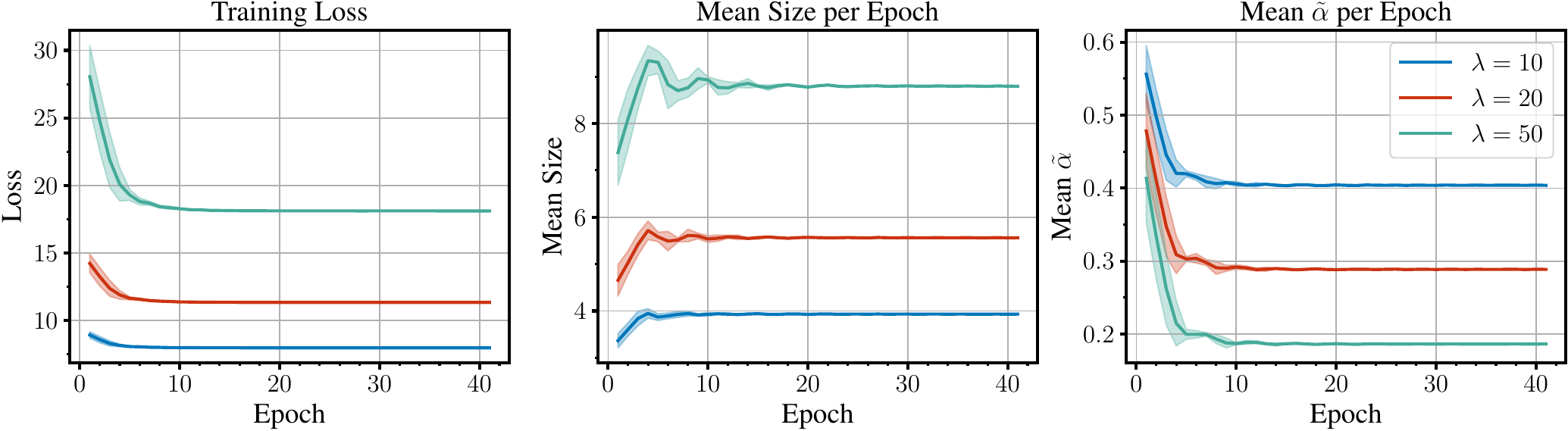}
    \caption{Training curves for $\lambda \in \{5,10,50\}$, averaged over 5 runs and smoothed with a moving average of size 10 for clarity. Shaded regions show $\pm1$ standard deviation across runs. \textbf{Left:} training loss. \textbf{Center:} mean set size. \textbf{Right:} mean adaptive $\tilde\alpha$.}
    \label{fig:reg_training}
\end{figure}

Once the model is trained, we can use it to produce conformal sets. As in the main experiment of the paper, we sample 100 test points. Here, we focus on the validity of Theorem \ref{thm:loo-size-consistency} in the regression setting. We consider a model trained with $\lambda = 50$. The leave-one-out estimator $\overline{\rm Size}_n$ is equal to 8.82, while the expected size of the test conformal set, computed from the 100 test conformal sets, is 9.31. The estimator is therefore relatively close to the true expected size. The gap between the two stems from the relatively small value of $n$ in our experiment. Empirically, we observe that this difference shrinks as $n$ increases, illustrating the effectiveness of the estimator.

Finally, we note an important practical limitation of this specific regression setup. As derived in Remark \ref{rmk:sets_reg}, the conformal set size for the MAE score using the soft-rank e-variable does not depend on the test feature~$X_{\rm test}$. Consequently, the resulting prediction sets maintain a constant width across the feature space and cannot adapt to varying levels of uncertainty or noise in the data. While this makes the sets overly large and rigid for practical, real-world regression tasks, this setup serves as a clean, differentiable testbed to empirically validate our procedure. For practical regression applications, combining our adaptive coverage framework with locally adaptive non-conformity scores or alternative e-variable constructions would be more appropriate.

\section{COMPLEXITY ANALYSIS AND RESOURCE DETAILS}

In our complexity analysis, we assume that the base model $f$ is a frozen black box. Therefore, the non-conformity scores $S(X_i, y)$ for all calibration points and all classes can be pre-computed prior to training. Once these scores are cached, querying $S$ inside the training loop becomes a constant-time memory lookup. Under this pre-computation, in the classification setting with the loss (\ref{eq:smooth_sigmoid}), the time complexity of the optimization loop in Algorithm \ref{algorithm} is
\[
O\!\left(T \cdot B \cdot (C_{\text{NN}} + n + |\mathcal{Y}|)\right),
\] 
where $C_{\text{NN}}$ is the cost of a forward pass through the network $\tilde{\alpha}_\theta$, $B$ is the batch size, $T$ is the number of training epochs, and $|\mathcal{Y}|$ is the number of classes. 
In the regression setting with the loss \eqref{eq:size_reg}, the time complexity reduces to 
\[
O\!\left(T \cdot B \cdot (C_{\text{NN}} + n)\right),
\] 
since conformal set sizes can be computed in constant time.

Algorithm \ref{algorithm_lambda} consists of a bracketing phase followed by bisection. In the bracketing phase, $\lambda$ is repeatedly doubled or halved until a bracket around the target mean set size $M$ is found, requiring $O(\log R)$ iterations if the initial $\lambda$ differs by a factor of $R$ from the solution. In the bisection phase, the interval is halved until the mean size is within a tolerance $\varepsilon$, yielding $O(\log(1/\varepsilon))$ iterations. Each iteration involves training the network for $T$ epochs on minibatches of size $B$ using Algorithm \ref{algorithm}, giving an overall complexity of:
\[
O\!\left( (\log R + \log(1/\varepsilon)) \cdot T \cdot B \cdot (C_{\text{NN}} + n + |\mathcal{Y}|)\right),
\]

All experiments were run on a machine with a 13th Gen Intel\textsuperscript{\textregistered} Core\texttrademark{} i7-13700H CPU. 
Training times ranged from a few seconds in the regression setting to a few minutes in the classification setting on CIFAR-10.

To address the scalability of our approach to domains substantially larger than CIFAR-10 (such as ImageNet), we note that Algorithm \ref{algorithm} operates exclusively on the pre-computed non-conformity scores. It is entirely independent of the raw input dimensionality (e.g., image resolution). As established above, the training complexity scales linearly with the calibration set size $n$ and the number of classes $|\mathcal{Y}|$. Even for massive datasets like ImageNet, the required calibration set size $n$ typically remains small and manageable (e.g., $n \approx 1000$). Thus, for a large-scale task ($n=1000$, $|\mathcal{Y}|=1000$), the neural coverage policy $\tilde{\alpha}_\theta$ is simply trained over low-dimensional score vectors (of size 1001). The linear scaling in $n$ and $|\mathcal{Y}|$: adds negligible computational overhead. The total wall-clock time is therefore overwhelmingly dominated by the initial forward passes of the base predictor $f$ needed to cache the scores, ensuring our adaptive conformal wrapper remains highly scalable to large real-world tasks on standard hardware.

\end{document}